\documentclass[twoside,11pt]{article}

\usepackage{jmlr2e}

\usepackage{graphicx}
\usepackage{subcaption}
\usepackage{booktabs} 
\usepackage{adjustbox}

\usepackage{amsmath}
\usepackage{yhmath}
\usepackage{stmaryrd}
\usepackage{pgfplots}
\usepgfplotslibrary{colorbrewer}
\usepgfplotslibrary{fillbetween}
\usepgfplotslibrary{groupplots}
\usepackage{makecell}
\usepackage{bbm}
\usepackage{enumerate}

\definecolor{mycolor8}{rgb}{0, 0, 1}
\definecolor{mycolor7}{rgb}{0.15, 0.15, 0.9}
\definecolor{mycolor6}{rgb}{0.3, 0.3, 0.75}
\definecolor{mycolor5}{rgb}{0.45, 0.45, 0.6}
\definecolor{mycolor4}{rgb}{0.6, 0.6, 0.45}
\definecolor{mycolor3}{rgb}{0.75, 0.75, 0.3}
\definecolor{mycolor2}{rgb}{0.9, 0.9, 0.15}
\definecolor{mycolor1}{rgb}{1, 1, 0}

\newcommand{\R}{\mathbb{R}}
\newcommand{\N}{\mathbb{N}}
\newcommand{\E}{\mathbb{E}}
\newcommand{\Ell}{\mathcal{L}}

\newcommand{\Id}{\mathcal{I}}
\newcommand{\dual}[1]{\widehat{\ #1 \ }}
\newcommand{\limiting}[1]{\overset{\scriptscriptstyle\infty}{#1}}

\DeclareMathOperator{\sgn}{sgn}

\usepackage{lastpage}
\jmlrheading{23}{2025}{1-\pageref{LastPage}}{1/21; Revised 5/22}{9/22}{21-0000}{D\'avid Terj\'ek and Diego Gonz\'alez-S\'anchez}

\ShortHeadings{MLPs at the EOC: Spectrum of the NTK}{Terj\'ek and Gonz\'alez-S\'anchez}
\firstpageno{1}

\begin{document}

\title{MLPs at the EOC: Spectrum of the NTK}

\author{
  \name D\'avid Terj\'ek\thanks{Corresponding author.}
  \email dterjek@renyi.hu \\
  \addr Alfr\'ed R\'enyi Institute of Mathematics \\ Budapest, Hungary
  \AND
  \name Diego Gonz\'alez-S\'anchez
  \email diegogs@renyi.hu \\
  \addr Alfr\'ed R\'enyi Institute of Mathematics \\ Budapest, Hungary
}

\editor{}

\maketitle

\begin{abstract}
We study the properties of the Neural Tangent Kernel (NTK) $\limiting{K} : \R^{m_0} \times \R^{m_0} \to \R^{m_l \times m_l}$ corresponding to infinitely wide $l$-layer Multilayer Perceptrons (MLPs) taking inputs from $\R^{m_0}$ to outputs in $\R^{m_l}$ equipped with activation functions $\phi(s) = a s + b \vert s \vert$ for some $a,b \in \R$ and initialized at the Edge Of Chaos (EOC). We find that the entries $\limiting{K}(x_1,x_2)$ can be approximated by the inverses of the cosine distances of the activations corresponding to $x_1$ and $x_2$ increasingly better as the depth $l$ increases. By quantifying these inverse cosine distances and the spectrum of the matrix containing them, we obtain tight spectral bounds for the NTK matrix $\limiting{K} = [\frac{1}{n} \limiting{K}(x_{i_1},x_{i_2}) : i_1, i_2 \in [1:n]]$ over a dataset $\{x_1,\cdots,x_n\} \subset \R^{m_0}$, transferred from the inverse cosine distance matrix via our approximation result. Our results show that $\Delta_\phi = \frac{b^2}{a^2+b^2}$ determines the rate at which the condition number of the NTK matrix converges to its limit as depth increases, implying in particular that the absolute value ($\Delta_\phi=1$) is better than the ReLU ($\Delta_\phi=\frac{1}{2}$) in this regard.
\end{abstract}


\section{Introduction} \label{introduction}
Formally introduced in the celebrated work of \citet{Jacotetal2018}, the NTK has been widely employed to analyze the problem of overparameterized learning. Given a neural network $N : \R^{m_0} \times \Theta \to \R^{m_l}$ that maps an input $x \in \R^{m_0}$ and a parameter $\theta \in \Theta$ to an output $N(x,\theta) \in \R^{m_l}$, the corresponding NTK at some parameter $\theta$ is the matrix-valued kernel $K_\theta : \R^{m_0} \times \R^{m_0} \to \R^{m_l \times m_l}$ defined as $K_\theta(x_1,x_2) = \partial_\theta N(x_1,\theta) {\partial_\theta N(x_2,\theta)}^*$ (the product of the Jacobian of $N(x_1,\cdot) : \Theta \to \R^{m_l}$ and the adjoint of the Jacobian of $N(x_2,\cdot) : \Theta \to \R^{m_l}$) for all input pairs $x_1,x_2 \in \R^{m_0}$. \citet{Jacotetal2018} showed that for MLPs, as width grows to infinity, $K_\theta$ at initialization (with $\theta$ drawn from the initial parameter distribution) converges in probability to a limiting NTK $K_\infty$.

Parallel to these developments, the study of infinitely deep neural networks led \citet{Pooleetal2016} to the discovery of the so-called Edge of Chaos (EOC). \citet{Schoenholzetal2017} showed that the EOC is the regime where infinitely deep MLPs avoid both exploding and vanishing gradients. In this regime, \citet{Hayouetal2019} described the asymptotic behavior of the cosines (correlations) of the activations in the infinitely wide limit, \citet{Xiaoetal2020} characterized the spectrum of $K_\infty$ by sending first the width and then depth to infinity, \citet{Hayouetal2022} quantified the entries of $K_\infty$ and \citet{Seleznovaetal22} studied the entries of both $K_\theta$ and $K_\infty$ when width and depth grow with a constant ratio.

The motivation for our work was to provide non-asymptotic quantitative bounds analogous to the results of \citet{Hayouetal2019} and \citet{Xiaoetal2020} for MLPs with $(a,b)$-ReLU activations $\phi(s) = a s + b \vert s \vert$, to be applied in \citet{mlpsateoc2} in order to prove that the NTK concentrates around its limit. To this end, we first quantify the propagation of the inverse cosine distances of the infinitely wide limits of the hidden activations in Proposition~\ref{prop:inverse_cosine_distance_propagation}. Then we prove that the entries of the limiting NTK are approximated better and better by the inverse cosine distances as depth increases in Proposition~\ref{prop:inverse_cosine_distances_approximate_ntk}. Using the convexity of the function that propagates the inverse cosine distances across the layers, we then obtain tight spectral bounds for the inverse cosine distance matrices in Proposition~\ref{prop:inverse_cosine_distance_matrices}. Finally, based on the approximation result we transfer the spectral bounds to the limiting NTK matrix in Theorem~\ref{thm:ntk_spectrum}. Our results quantify the effect of the parameters $(a,b)$, showing that it is beneficial to the spectral properties of the limiting NTK matrix if $\Delta_\phi = \frac{b^2}{a^2+b^2} \in [0,1]$ (the squared cosine of the angle next to the second leg in the right triangle with legs of length $a$ and $b$) is close to $1$. In particular, we have $\Delta_\phi=0$ for the identity, $\Delta_\phi=\frac{1}{2}$ for the ReLU and $\Delta_\phi=1$ for the absolute value, with the latter leading to the best possible conditioning for a given depth $l$ (agreeing with the findings of \citet{Yangetal2024} that the absolute value is optimal among homogeneous activations).

The organization of the rest of the paper is as follows. We conclude \S~\ref{introduction} by listing our contributions in \S~\ref{contributions} and introduce some notation in \S~\ref{preliminaries}. We state and prove our results concerning the limiting NTK in \S~\ref{limiting_ntk}. We conclude by discussing the limitations of our work in \S~\ref{limitations} along with future directions.

\subsection{Contributions} \label{contributions}
We propose
\begin{itemize}
\item non-asymptotic, quantitative bounds for the cosines of infinitely wide activations,
\item an approximation of the NTK entries in terms of inverse cosine distances,
\item tight spectral bounds for the limiting NTK matrix and
\item a quantification of the effect that the choice of homogeneous activation function has on the spectrum of the NTK.
\end{itemize}

\section{Preliminaries}\label{preliminaries}

Given $i, j \in \N$, we define the tuple $[i:j] = (i,i+1,\cdots,j-1,j)$ (which is the empty tuple $()$ if $i > j$). For any $m, n \in \N$, we denote by $m\N+n$ the set $\{mr+n : r \in \N\}$. We denote by $\Vert \cdot \Vert$ both the Euclidean norm on $\R^n$ and the operator norm on $\R^{n \times n}$. On the latter, we denote the Frobenius norm by $\Vert \cdot \Vert_F$ and the infinity norm by $\Vert \cdot \Vert_\infty$ (with the latter defined as $\Vert A \Vert_\infty = \max_{i_1 \in [1:n]}\{ \sum_{i_2 \in [1:n]} \vert A_{i_1,i_2} \vert \}$). The adjoint of a matrix $A \in \R^{n \times n}$ is the unique matrix $A^* \in \R^{n \times n}$ such that $\langle A x_1, x_2 \rangle = \langle x_1, A^* x_2 \rangle$ for all $x_1,x_2 \in \R^n$. We denote the set of $n \times n$ symmetric matrices by $\mathbb{S}^n = \{ A \in \R^{n \times n} : A = A^*\}$ and the set of $n \times n$ symmetric positive semidefinite matrices by $\mathbb{S}^n_+ = \{ A \in \mathbb{S}^n : \langle x, A x \rangle \geq 0 \text{ for } \forall x \in \R^n \}$. For $A \in \mathbb{S}^n$, we denote the $i$th eigenvalue by $\lambda_i(A)$ with the order being descending as $\lambda_1(A) \geq \cdots \geq \lambda_n(A)$, the smallest and largest eigenvalues by $\lambda_{\min}(A)=\lambda_n(A)$ and $\lambda_{\max}(A)=\lambda_1(A)=\Vert A \Vert$, respectively, as well as the condition number by $\kappa(A) = \frac{\lambda_{\max}(A)}{\lambda_{\min}(A)}$ provided that $\lambda_{\min}(A)\not=0$. Note that by the Gershgorin circle theorem we have $\Vert A \Vert \leq \Vert A \Vert_\infty$ for any $A \in \mathbb{S}^n$. We denote by $\Id_n \in \R^{n \times n}$ the identity matrix on $\R^n$. We denote the tensor product of a pair of vectors $x_1, x_2 \in \R^n$ by $x_1 \otimes x_2 = [ {x_1}_{i_1} {x_2}_{i_2} : i_1,i_2 \in [1:n]] \in \R^{n \times n}$ and the second tensor power of a vector $x \in \R^n$ by $x^{\otimes 2} = x \otimes x \in \mathbb{S}^n_+$. We denote the $n$-dimensional constant $1$ vector by $\mathbbm{1}_n = [ 1 : i \in [1:n]] \in \R^n$. For matrices $A_1 \in \R^{n \times n}$ and $A_2 \in \R^{m \times m}$, we denote their Kronecker product $A_1 \boxtimes A_2 = [ {A_1}_{i_1,i_2} A_2 : i_1,i_2 \in [1:n]] \in \R^{nm \times nm}$. Given $x \in \R^n$, we define the corresponding diagonal matrix $D_x \in \mathbb{S}^n$ as ${D_x}_{i_1,i_2} = x_i$ if $i_1 = i_2 = i$ and $0$ otherwise for all $i_1,i_2 \in [1:n]$.

The infinity and Lipschitz norms of real-valued functions are denoted by $\Vert \cdot \Vert_\infty$ and $\Vert \cdot \Vert_L$, respectively. Given a function $F:G \to H$, we say that it is differentiable if it is Fr\'echet differentiable, i.e., if there exists a bounded linear operator $\partial F(x) \in \Ell(G,H)$, which we refer to as the Jacobian of $F$ at $x$, satisfying $\lim_{y \to x }\frac{\Vert F(y) - F(x) - \partial F(x) (y - x) \Vert}{\Vert y - x \Vert}=0$. For a function $f$ with the same domain and codomain, we denote by $f^{\circ n}$ the nested composition of $f$ with itself $n \in \N$ times, with $f^{\circ 0}$ being the identity. We use the $O(\cdot)$ and $\Omega(\cdot)$ asymptotic notation in the sense that for functions $f,g : \N \to \R_+$, we say that $f(n) = O(g(n))$ (resp. $f(n) = \Omega(g(n))$) if there exists implicit constants $C \in \R_+$ and $n_0 \in \N$ such that $f(n) \leq C g(n)$ (resp. $f(n) \geq C g(n)$) for all $n \geq n_0$. The notation $f = \Theta(g)$ means that both $f=O(g)$ and $f=\Omega(g)$ hold. We denote the digamma function by $\digamma(x)=\frac{\Gamma'(x)}{\Gamma(x)}$.

Given $\mu \in \R^n$ and $\Sigma \in \mathbb{S}^n_+$, we denote by $\mathcal{N}\left(\mu,\Sigma\right)$ the multivariate Gaussian distribution with mean $\mu$ and covariance $\Sigma$. In particular, $\mathcal{N}(0,1)$ is the standard Gaussian distribution. By $X \sim \mathcal{N}(\mu, \Sigma)$ we mean that the random vector $X$ is distributed according to $\mathcal{N}(\mu, \Sigma)$. We use the same notation to denote the corresponding probability measure, i.e., $\E_{X \sim \mathcal{N}(\mu, \Sigma)} f(X) = \int f d\mathcal{N}(\mu, \Sigma) = \int f(x) d\mathcal{N}(x \vert \mu, \Sigma)$. We denote the norm of the Hilbert space $L^2(\mathcal{N}(0,1))$ by $\Vert f \Vert_{\mathcal{N}(0,1)} = \sqrt{\int f^2 d\mathcal{N}(0,1)}$ for $f \in L^2(\mathcal{N}(0,1))$.

\section{Infinite Width NTK at the EOC}\label{limiting_ntk}

We define $(a,b)$-ReLU activations as follows.

\begin{definition}[$(a,b)$-ReLU]~\\
Given $a,b \in \R$, define the $(a,b)$-ReLU $\phi : \R \to \R$ for all $s \in \R$ as $\phi(s) = as + b\vert s \vert$, so that $\phi'(s) = a + b \sgn(s)$ for all $s \in \R \setminus \{ 0 \}$ and $\Vert \phi \Vert_L = \vert a \vert + \vert b \vert$.
\end{definition}
Unless $b=0$, $\phi$ is not differentiable at $s=0$ in the usual sense, but any function $\psi : \R \to \R$ such that $\psi(s) = a + b \sgn(s)$ for all $s \in \R \setminus \{ 0 \}$ and $\psi(0) \in [a-b,a+b]$ can serve as its derivative in some suitable generalized sense. By abuse of notation, we define $\phi': \R \to \R$ as $\phi'(s) = a + b \sgn(s)$ for all $s \in \R$, so that $\phi'(0) = a$.

We do not treat the infinitely wide limits of the hidden layer activations and their derivatives directly, accessing them only via their inner products as defined below. The only hyperparameter of the MLP parameterization in \citet{mlpsateoc2} that affects the limiting NTK besides $(a,b)$ and the depth $l$ is the initial variance $\sigma^2$.

\begin{definition}[Limiting activation and derivative inner products]~\\
Define $\limiting{X}_k : \R^{m_0} \times \R^{m_0} \to \R$ for all $x_1,x_2 \in \R^{m_0}$ recursively as $\limiting{X}_1(x_1,x_2) = \langle x_1, x_2 \rangle$ and
\[
\limiting{X}_k(x_1,x_2) = \sigma^2 \int \phi(u_1) \phi(u_2) d\mathcal{N}\left( [u_1,u_2] \left\vert 0,\left[ \begin{smallmatrix} \limiting{X}_{k-1}(x_1,x_1) & \limiting{X}_{k-1}(x_1,x_2) \\ \limiting{X}_{k-1}(x_1,x_2) & \limiting{X}_{k-1}(x_2,x_2) \end{smallmatrix} \right] \right. \right) \text{ for } k \in [2:l]
\]
and define $\limiting{X'}_k : \R^{m_0} \times \R^{m_0} \to \R$ for all $x_1,x_2 \in \R^{m_0}$
\[
\limiting{X'}_k(x_1,x_2) = \sigma^2 \int \phi'(u_1) \phi'(u_2) d\mathcal{N}\left( [u_1,u_2] \left\vert 0,\left[ \begin{smallmatrix} \limiting{X}_{k-1}(x_1,x_1) & \limiting{X}_{k-1}(x_1,x_2) \\ \limiting{X}_{k-1}(x_1,x_2) & \limiting{X}_{k-1}(x_2,x_2) \end{smallmatrix} \right] \right. \right) \text{ for } k \in [2:l].
\]
\end{definition}

With these in hand, we can define the limiting NTK.

\begin{definition}[Limiting NTK]~\\
Define $\limiting{K} : \R^{m_0} \times \R^{m_0} \to \R^{m_l \times m_l}$ for all $x_1,x_2 \in \R^{m_0}$ as
\[
\limiting{K}(x_1,x_2)
= \left( \sum_{k=1}^l \limiting{X}_k(x_1,x_2) \prod_{k'=k+1}^l \limiting{X'}_{k'}(x_1,x_2) \right) \Id_{m_l}.
\]
\end{definition}

For $(a,b)$-ReLUs, the EOC initialization is provided by \citet[Lemma~3]{Hayouetal2019} (referred to as the weak EOC in \citet{Hayouetal2019}).

\begin{definition}[EOC initialization for $(a,b)$-ReLUs]\label{def:eoc_parameterization}~\\
The MLP parameterization of \citet{mlpsateoc2} is initialized at the EOC by setting
\[
\sigma = (a^2 + b^2)^{-\frac{1}{2}}.
\]
\end{definition}

The following definition is taken from \citet{Danielyetal2016}.

\begin{definition}[Dual function]~\\
Given $f \in L^2(\mathcal{N}(0,1))$, define $\dual{f} : [-1,1] \to \R$ as
\[
\dual{f}(\rho) 
= \int f(u_1) f(u_2) d\mathcal{N}\left( [u_1,u_2] \left\vert 0,\left[ \begin{smallmatrix} 1 & \rho \\ \rho & 1 \end{smallmatrix} \right] \right. \right)\text{ for all } \rho \in [-1,1].
\]
\end{definition}

\begin{remark}[$\vert \cdot \vert$ and $\sgn$]\label{rem:abs_sgn_properties}~\\
Although the absolute value is not differentiable at $0$, we do have that $\vert \cdot \vert' \in L^2(\mathcal{N}(0,1))$ is the equivalence class of the sign function $\sgn$. Therefore, by abuse of notation we refer to $\sgn$ as the derivative of $\vert \cdot \vert$. Additionally, we have
\[
\Vert \vert \cdot \vert \Vert_{\mathcal{N}(0,1)} = \Vert \sgn \Vert_{\mathcal{N}(0,1)} = \Vert \vert \cdot \vert \Vert_L = \Vert \sgn \Vert_\infty = 1, 
\]
\[
\dual{\vert\cdot\vert}(\rho) = \frac{2}{\pi}(\rho\arcsin(\rho) + \sin(\arccos(\rho)),
\]
\[
\dual{\vert\cdot\vert}'(\rho) = \dual{\sgn}(\rho) = \frac{2}{\pi}\arcsin(\rho)
\]
and
\[
\dual{\vert\cdot\vert}''(\rho) = \dual{\sgn}'(\rho)=\frac{2}{\pi}(1-\rho^2)^{-\frac{1}{2}}.
\]
\end{remark}

Denote the cosines of the limiting activations for all $x_1,x_2 \in \R^{m_0}$ and $k \in [1:l]$ as
\[
\limiting{\rho}_k(x_1,x_2) = \frac{\limiting{X}_k(x_1,x_2)}{\sqrt{\limiting{X}_k(x_1,x_1)} \sqrt{\limiting{X}_k(x_2,x_2)}}.
\]

\begin{proposition}[Limiting inner products as dual functions]\label{prop:limiting_inner_products}~\\
For all $x_1, x_2 \in \R^{m_0}$ and $k \in [2:l]$, we have
\[
\limiting{X}_k(x_1,x_2)
= \sqrt{\limiting{X}_{k-1}(x_1,x_1)} \sqrt{\limiting{X}_{k-1}(x_2,x_2)} \sigma^2 \dual{\phi}\left( \limiting{\rho}_{k-1}(x_1,x_2) \right)
\]
and
\[
\limiting{X'}_k(x_1,x_2)
= \sigma^2 \dual{\phi'}\left( \limiting{\rho}_{k-1}(x_1,x_2) \right).
\]
\end{proposition}
\begin{proof}
Both claims follow from the definitions using the homogeneity of $\phi$ and the fact that $\phi'(s) = \phi'(ts)$ for all $t > 0$.
\end{proof}

\begin{proposition}[Limiting inner products at the EOC]\label{prop:eoc_properties}~\\
Setting \eqref{def:eoc_parameterization}, for all $x \in \R^{m_0}$ we have $\limiting{X}_k(x,x) = \Vert x \Vert^2$ for all $k \in [1:l]$ and $\limiting{X'}_k(x,x) = 1$ for all $k \in [2:l]$.
\end{proposition}
\begin{proof}
Note that $\sigma = \Vert \phi \Vert_{\mathcal{N}(0,1)}^{-1} = \Vert \phi' \Vert_{\mathcal{N}(0,1)}^{-1}$ since $\vert\cdot\vert$ is an even function with $\Vert \vert\cdot\vert \Vert_{\mathcal{N}(0,1)}^2=1$ and $\sgn$ is odd with $\Vert \sgn \Vert_{\mathcal{N}(0,1)}^2=1$, while we also have $\dual{\phi}(1) = \dual{\phi'}(1) = \Vert \phi \Vert_{\mathcal{N}(0,1)} = \Vert \phi' \Vert_{\mathcal{N}(0,1)}$. Assume that $\limiting{X}_{k'}(x,x) = \Vert x \Vert^2$ for all $k' \in [1:k-1]$, which holds by definition for $k=2$. Therefore via Proposition~\ref{prop:limiting_inner_products} we have $\limiting{X}_k(x,x) = \limiting{X}_{k-1}(x,x) \sigma^2 \Vert \phi \Vert_{\mathcal{N}(0,1)}^2 = \limiting{X}_{k-1}(x,x)$, giving the first claim by induction. By Proposition~\ref{prop:limiting_inner_products} we have $\limiting{X'}_k(x,x) = \sigma^2 \Vert \phi' \Vert_{\mathcal{N}(0,1)}^2 = 1$, giving the second claim.
\end{proof}

Note that the cosine of two random variables is their correlation. In fact, what we refer to as cosines are usually called correlations in the EOC literature, e.g. in \citet{Hayouetal2019}. Nevertheless, we stick to the former as we are going to analyze them via the associated cosine distances, aligning more with the geometric viewpoint than the probabilistic one.

\begin{proposition}[Cosine map]\label{prop:cosine_map}~\\
Denote $\Delta_\phi = \frac{b^2}{a^2+b^2}$ and define the cosine map $\varrho : [-1,1] \to \R$ as $\varrho(\rho) = \sigma^2 \dual{\phi}(\rho)$ for all $\rho \in [-1,1]$. Given $x_1,x_2 \in \R^{m_0}$ and setting \eqref{def:eoc_parameterization}, we have $\limiting{\rho}_k(x_1,x_2) = \varrho(\limiting{\rho}_{k-1}(x_1,x_2))$ and $\limiting{X'}_k(x_1,x_2) = \varrho'(\limiting{\rho}_{k-1}(x_1,x_2))$ for all $k \in [2:l]$. Moreover, we have
\[
\varrho(\rho) 
= \rho + \Delta_\phi \frac{2}{\pi}\left( \sqrt{1-\rho^2} - \rho\arccos(\rho) \right) \in [-1,1]
\]
and
\[
\varrho'(\rho) 
= 1 - \Delta_\phi \frac{2}{\pi}\arccos(\rho) \in [-1,1] 
\] 
for all $\rho \in [-1,1]$. 
\end{proposition}
\begin{proof}
Proposition~\ref{prop:limiting_inner_products} implies the expressions for $\limiting{\rho}_k(x_1,x_2)$ and $\limiting{X'}_k(x_1,x_2)$. As $\dual{a \cdot}(\rho) = a^2 \rho$ and $\dual{a}(\rho) = a^2$, Proposition~\ref{prop:limiting_inner_products} and Remark~\ref{rem:abs_sgn_properties} together yield the formulas for $\varrho(\rho)$ and $\varrho'(\rho)$ (using that $\Vert \vert\cdot\vert \Vert_{\mathcal{N}(0,1)} = \Vert \sgn \Vert_{\mathcal{N}(0,1)} = 1$ and the fact that $\phi$ and $\phi'$ are both linear combinations of an odd and an even function).
\end{proof}

The limiting NTK takes the following a simple form for $(a,b)$-ReLUs at the EOC.

\begin{proposition}[$\limiting{K}$ at the EOC]\label{prop:limiting_ntk_at_eoc}~\\
Given $x_1,x_2 \in \R^{m_0}$ and setting \eqref{def:eoc_parameterization}, we have
\[
\limiting{K}(x_1,x_2) = \Vert x_1 \Vert \Vert x_2 \Vert \left( \sum_{k=1}^l \varrho^{\circ (k-1)}\left( \limiting{\rho}_1(x_1,x_2) \right) \prod_{k'=k}^{l-1} \varrho'\left( \varrho^{\circ (k'-1)}\left( \limiting{\rho}_1(x_1,x_2) \right) \right) \right) \Id_{m_l}.
\]
\end{proposition}
\begin{proof}
The claim follows from the definitions and Proposition~\ref{prop:cosine_map}.
\end{proof}
Note that in the last term, we have the empty product $\prod_{k'=l}^{l-1} \varrho'(\varrho^{\circ (k'-1)}(\limiting{\rho}_1(x_1,x_2))) = 1$.

To a cosine $\rho \in [-1,1]$, we associate the squared cosine distance $z = \frac{1-\rho}{2} \in [0,1]$ and (if $z \neq 0$) the inverse cosine distance $w = z^{-\frac{1}{2}} = \left( \frac{1-\rho}{2} \right)^{-\frac{1}{2}} \in [1,\infty)$. The naming is justified by the fact that if $\rho = \langle y_1, y_2 \rangle$ for $y_1,y_2 \in \R^n$ with $\Vert y_1 \Vert = \Vert y_2 \Vert = 1$, then $z = \Vert \frac{1}{2} y_1 - \frac{1}{2} y_2 \Vert^2$.

\begin{proposition}[Squared cosine distance map]\label{prop:squared_cosine_distance_map}~\\
Define the squared cosine distance map $\zeta : [0,1] \to [0,1]$ as $\zeta(z) = \frac{1-\varrho(1 - 2z)}{2}$ for $z \in [0,1]$. Then we have
\[
\zeta(z) 
= z - \Delta_\phi \frac{1}{\pi} \left( \sqrt{1-(1-2z)^2} - (1-2z)\arccos(1-2z) \right) 
= z - \Delta_\phi \sum_{r \in 2\N+3} b_r z^{\frac{r}{2}},
\]

\[
\zeta'(z) 
= 1 - \Delta_\phi \frac{2}{\pi} \arccos(1-2z) 
= 1 - \Delta_\phi \sum_{r \in 2\N+3} \frac{1}{2} r b_r z^{\frac{r-2}{2}}
\]
and
\[
\zeta''(z)
= - \Delta_\phi \frac{2}{\pi} (1-z)^{-\frac{1}{2}} z^{-\frac{1}{2}}
= - \Delta_\phi \sum_{r \in 2\N+3} \frac{1}{4} r (r-2) b_r z^{\frac{r-4}{2}}
\]
with the coefficients
\[
b_r = \frac{2}{\pi} \frac{((\frac{r}{r-2})^2-1)((r-2)!!)^2}{r!} \text{ for } r \in 2\N+3
\]
satisfying $b_3 = \frac{8}{3\pi}$, $b_r \geq 0$ and $\sum_{r \in 2\N+3} b_r = 1$.
\end{proposition}
\begin{proof}
The power series expansions 
\[
\arccos(1-2z)=\sum_{r \in 2\N+1} 2 \frac{((r-2)!!)^2}{r!} z^{\frac{r}{2}}
\]
and 
\[
\sqrt{1-(1-2z)^2}=\sum_{r \in 2\N+1} -2 \frac{(r-4)!! r!!}{r!} z^{\frac{r}{2}}
\]
give the claims.
\end{proof}
Note that we have $\varrho^{\circ k}(\rho) = 1 - 2 \zeta^{\circ k}(\frac{1-\rho}{2})$ and $\varrho'(\rho) = \zeta'(\frac{1-\rho}{2})$ for all $\rho \in [-1,1]$.

\begin{proposition}[Inverse cosine distance map]\label{prop:inverse_cosine_distance_map}~\\
Define the inverse cosine distance map $\omega : (1,\infty) \to (1,\infty)$ as $\omega(w) = \zeta(w^{-2})^{-\frac{1}{2}}$ for $w \in (1,\infty)$. Then $\omega$ is convex,
\[
\omega(w) = w + \Delta_\phi \frac{4}{3\pi} + \frac{3}{2} \left( \Delta_\phi \frac{4}{3\pi} \right)^2 w^{-1} + \Delta_\phi \varepsilon(w) w^{-2}
\]
with $\varepsilon : (1,\infty) \to (0,\infty)$ strictly decreasing and $\omega'(w) \in [0,1)$ for $w \geq \left( \frac{1-\cos(\min\{ \Delta_\phi^{-1} \frac{\pi}{2}, \pi\})}{2} \right)^{-\frac{1}{2}}$.
\end{proposition}
\begin{proof}
Note that $\omega(w) = w (1 - g(w))^{-\frac{1}{2}}$ with $g(w) = 1 - \frac{\zeta(w^{-2})}{w^{-2}} = \Delta_\phi \sum_{r \in 2\N+3} b_r w^{-(r-2)}$. As $\sum_{k \in \N} \frac{(2k-1)!!}{2^k k!} s^k = (1-s)^{-\frac{1}{2}}$ for all $s \in [0,1]$, we have $\omega(w) = w \left( \frac{\zeta(w^{-2})}{w^{-2}} \right)^{-\frac{1}{2}} = w \sum_{k \in \N} \frac{(2k-1)!!}{2^k k!} g(w)^k$. Expanding this expression, we then have $\omega(w) = w + \frac{1}{2} \Delta_\phi b_3 + \frac{3}{2} ( \frac{1}{2} \Delta_\phi b_3 )^2 w^{-1} + \Delta_\phi \varepsilon(w) w^{-2}$ with
\begin{multline*}
\varepsilon(w) = \left( \frac{1}{2} + \frac{3}{4} \Delta_\phi b_3 w^{-1} + \frac{3}{8} \Delta_\phi \sum_{r \in 2\N+5} b_r w^{-(r-2)} \right) \sum_{r \in 2\N+5} b_r w^{-(r-5)} \\
+ \sum_{k \in \N+3} \frac{(2k-1)!!}{2^k k!} \left( \sum_{r \in 2\N+3} b_r w^{-(r-3)} \right)^k \Delta_\phi^{k-1} w^{-(k-3)}.
\end{multline*}
We can write $\varepsilon(w) = \sum_{k \in \N} c_k w^{-k}$ with $c_0 = \frac{1}{2} b_5 + \Delta_\phi^2 \frac{5}{16} b_3^3$ and $c_k \geq 0$ for $k \in \N+2$, so that $\varepsilon$ is strictly decreasing, $\varepsilon(w) \geq c_0$ and $\varepsilon'(w) = -\sum_{k \in \N+1} k c_k w^{-(k+1)}$. Taking the derivative, we get $\omega'(w) = 1 - \frac{3}{2} ( \frac{1}{2} \Delta_\phi b_3 )^2 w^{-2} + \Delta_\phi \varepsilon'(w) w^{-2} - 2 \Delta_\phi \varepsilon(w) w^{-3}$, so that $\omega' \leq 1$ and $\omega'$ is strictly increasing, hence $\omega$ is convex. We also have that $\omega'(w) = w^{-3} \zeta(w^{-2})^{-\frac{3}{2}} \zeta'(w^{-2}) = (1-g(w))^{-\frac{3}{2}} \zeta'(w^{-2})$ implying that $\omega'(w) \geq 0$ if $\zeta'(w^{-2}) \geq 0$, which holds if and only if $\rho = 1 - 2w^{-2} \geq \cos(\min\{ \Delta_\phi^{-1} \frac{\pi}{2}, \pi\})$ by Proposition~\ref{prop:cosine_map}, i.e., if $w \geq (\frac{1-\cos(\min\{ \Delta_\phi^{-1} \frac{\pi}{2}, \pi\})}{2})^{-\frac{1}{2}}$.
\end{proof}
Note that we have $\varrho^{\circ k}(\rho) = 1 - 2 \omega^{\circ k}\left( \left( \frac{1-\rho}{2} \right)^{-\frac{1}{2}} \right)^{-2}$ for all $\rho \in [-1,1]$.

Now we prove a non-asymptotic quantitative analogue of \citet[Proposition~1]{Hayouetal2019} and \citet[Theorem~1]{Hayouetal2022}, quantifying the propagation of inverse cosine distances across depth.

\begin{proposition}[Propagation of inverse cosine distances]\label{prop:inverse_cosine_distance_propagation}~\\
Given $w \in (1,\infty)$, for $k \in \N$ we have the bound
\begin{equation}\label{eq:w_j_bound}
\left\vert \omega^{\circ k}(w) - \left( w + \Delta_\phi \frac{4}{3\pi} (k-1) + \Delta_\phi \frac{2}{\pi} \log\left( \Delta_\phi^{-1} \frac{3\pi}{4} w + k - 1 \right) \right) \right\vert \leq O(1).
\end{equation}
\end{proposition}
\begin{proof}
Denote $w_k = \omega^{\circ (k-1)}(w)$ and $k_0 = (\frac{1}{2} \Delta_\phi b_3)^{-1} w_1 - 1$. We will apply Proposition~\ref{prop:inverse_cosine_distance_map} repeatedly. Clearly we have $w_k \geq \frac{1}{2} \Delta_\phi b_3 (k_0+k)$. This implies that
\begin{multline*}
w_k 
\leq \frac{1}{2} \Delta_\phi b_3 (k_0+k) + \frac{3}{2} \left( \frac{1}{2} \Delta_\phi b_3 \right)^2 \sum_{k'=1}^{k-1} w_{k'}^{-1} + \Delta_\phi \varepsilon(w_1) \sum_{k'=1}^{k-1} w_{k'}^{-2} \\
\leq \frac{1}{2} \Delta_\phi b_3 (k_0+k) + \frac{3}{2} \left( \frac{1}{2} \Delta_\phi b_3 \right)^2 \sum_{k'=1}^{k-1} \left( \frac{1}{2} \Delta_\phi b_3 (k_0+k) \right)^{-1} + \Delta_\phi \varepsilon(w_1) \sum_{k'=1}^{k-1} \left( \frac{1}{2} \Delta_\phi b_3 (k_0+k) \right)^{-2} \\
= \frac{1}{2} \Delta_\phi b_3 (k_0+k) + \frac{3}{4} \Delta_\phi b_3 (\digamma(k_0+k)-\digamma(k_0+1)) + \Delta_\phi \varepsilon(w_1) \left( \frac{1}{2} \Delta_\phi b_3 \right)^{-2} (\digamma'(k_0+1)-\digamma'(k_0+k)) \\
\leq \frac{1}{2} \Delta_\phi b_3 (k_0+k) + \frac{3}{4} \Delta_\phi b_3 \log(k_0+k) + O(1),
\end{multline*}
where we used the bounds $\log(s)-s^{-1} \leq \digamma(s) \leq \log(s)-\frac{1}{2}s^{-1}$ and $s^{-1} + \frac{1}{2} s^{-2} \leq \digamma'(s) \leq s^{-1} + s^{-2}$. 

As $(\frac{1}{2} \Delta_\phi b_3 (k_0+k))^{-1} - (\frac{1}{2} \Delta_\phi b_3 (k_0+k) + \frac{3}{4} \Delta_\phi b_3 \log(k_0+k) + O(1))^{-1} \leq (\frac{1}{2} \Delta_\phi b_3 (k_0+k))^{-2} (\frac{3}{4} \Delta_\phi b_3 \log(k_0+k) + O(1))$, we have
\[
w_k^{-1} \geq \left( \frac{1}{2} \Delta_\phi b_3 (k_0+k) \right)^{-1} - \left( \frac{1}{2} \Delta_\phi b_3 (k_0+k) \right)^{-2} \left( \frac{3}{4} \Delta_\phi b_3 \log(k_0+k) + O(1) \right),
\]
so that
\begin{multline*}
w_k 
\geq \frac{1}{2} \Delta_\phi b_3 (k_0+k) + \frac{3}{2} \left( \frac{1}{2} \Delta_\phi b_3 \right)^2 \sum_{k'=1}^{k-1} w_{k'}^{-1} \\
\geq \frac{1}{2} \Delta_\phi b_3 (k_0+k) + \frac{3}{2} \left( \frac{1}{2} \Delta_\phi b_3 \right)^2 \sum_{k'=1}^{k-1} \left( \frac{1}{2} \Delta_\phi b_3 (k_0+k') \right)^{-1} \\
- \frac{3}{2} \left( \frac{1}{2} \Delta_\phi b_3 \right)^2 \sum_{k'=1}^{k-1} \left( \frac{1}{2} \Delta_\phi b_3 (k_0+k') \right)^{-2} \left( \frac{3}{4} \Delta_\phi b_3 \log(k_0+k') + O(1) \right) \\
\geq \frac{1}{2} \Delta_\phi b_3 (k_0+k) + \frac{3}{4} \Delta_\phi b_3 \log(k_0+k) - O(1),
\end{multline*}
where we used that $\sum_{k'=1}^{k-1} (k_0+k')^{-2} \log(k_0+k') \leq O(1)$.
\end{proof}

The following result shows that the sum determining $\limiting{K}(x_1,x_2)$ in Proposition~\ref{prop:limiting_ntk_at_eoc} is approximately an affine function of the inverse cosine distance $\left( \frac{1 - \limiting{\rho}_l(x_1,x_2)}{2} \right)^{-\frac{1}{2}}$.

\begin{proposition}[Inverse cosine distances approximate the entries of the NTK]\label{prop:inverse_cosine_distances_approximate_ntk}~\\
Given $\rho \in (-1,1)$ and letting
\[
u_k = \sum_{k'=1}^k \varrho^{\circ (k'-1)}(\rho) \prod_{k''=k'}^{k-1} \varrho'(\varrho^{\circ (k''-1)}(\rho))
\]
for $k \in \N+1$, we have the bounds
\[
\Delta_\phi^{-1} \frac{3\pi}{16} \omega^{\circ (k-1)}\left( \left( \frac{1-\rho}{2} \right)^{-\frac{1}{2}} \right) - \frac{1}{8} - O(\Delta_\phi^{-2} k^{-1})
\leq u_k
\leq \Delta_\phi^{-1} \frac{3\pi}{16} \omega^{\circ (k-1)}\left( \left( \frac{1-\rho}{2} \right)^{-\frac{1}{2}} \right) - \frac{1}{8}.
\]
\end{proposition}
\begin{proof}
Letting $\rho_1 = \rho$, $\rho_k = \varrho(\rho_{k-1})$ and $z_k = \frac{1-\rho_k}{2}$, note that $z_k = \zeta(z_{k-1})$, $\rho_k = 1 - 2 z_k$, $\varrho'(\rho_k) = \zeta'(z_k)$ and $z_k^{-\frac{1}{2}} = \omega^{\circ (k-1)}(z_1^{-\frac{1}{2}})$. Assume that $u_k \leq \frac{1}{4} (\frac{1}{2} \Delta_\phi b_3)^{-1} z_k^{-\frac{1}{2}} - \frac{1}{8}$, which clearly holds for $k=1$. We then have
\begin{multline*}
u_{k+1} = \zeta'(z_k) u_k + 1 - 2 \zeta(z_k)
\leq \left( 1 - 3 \frac{1}{2} \Delta_\phi b_3 z_k^{\frac{1}{2}} \right) \left( \frac{1}{4} \left( \frac{1}{2} \Delta_\phi b_3 \right)^{-1} z_k^{-\frac{1}{2}} - \frac{1}{8} \right) + 1 \\
= \frac{1}{4} \left( \frac{1}{2} \Delta_\phi b_3 \right)^{-1} z_k^{-\frac{1}{2}} + \frac{1}{4} + \frac{3}{8} \frac{1}{2} \Delta_\phi b_3 z_k^{\frac{1}{2}} - \frac{1}{8} \\
= \frac{1}{4} \left( \frac{1}{2} \Delta_\phi b_3 \right)^{-1} \left( z_k^{-\frac{1}{2}} + s + \frac{3}{2} \left( \frac{1}{2} \Delta_\phi b_3 \right)^2 z_k^{\frac{1}{2}} \right) - \frac{1}{8},
\end{multline*}
so that $u_{k+1} \leq \frac{1}{4} (\frac{1}{2} \Delta_\phi b_3)^{-1} z_{k+1}^{-\frac{1}{2}} - \frac{1}{8}$, completing the induction.

Assume now that $u_k \geq \frac{1}{4} (\frac{1}{2} \Delta_\phi b_3)^{-1} z_k^{-\frac{1}{2}} - \frac{1}{8} - c k z_k$ with $c \geq 2 + \frac{1}{4} (\frac{1}{2} \Delta_\phi b_3)^{-1} z_1^{-\frac{3}{2}} - \frac{9}{8} z_1^{-1}$, which clearly holds for $k=1$. We then have
\begin{multline*}
u_{k+1} = \zeta'(z_k) u_k + 1 - 2 \zeta(z_k)
\geq \zeta'(z_k) \left( \frac{1}{4} \left( \frac{1}{2} \Delta_\phi b_3 \right)^{-1} z_k^{-\frac{1}{2}} - \frac{1}{8} - c k z_k \right) + 1 - 2 \zeta(z_k) \\
\geq \frac{1}{4} \left( \frac{1}{2} \Delta_\phi b_3 \right)^{-1} z_k^{-\frac{1}{2}} - \frac{1}{8} - 3 \frac{1}{2} \Delta_\phi b_3 z_k^{\frac{1}{2}} \left( \frac{1}{4} \left( \frac{1}{2} \Delta_\phi b_3 \right)^{-1} z_k^{-\frac{1}{2}} - \frac{1}{8} \right) \\
- \left( \Delta_\phi \sum_{r \in 2\N+5} \frac{1}{2} r b_r z_k^{\frac{r-2}{2}} \right) \left( \frac{1}{4} \left( \frac{1}{2} \Delta_\phi b_3 \right)^{-1} z_k^{-\frac{1}{2}} - \frac{1}{8} \right) - c k z_k \zeta'(z_k) + 1 - 2 \zeta(z_k) \\
\geq \frac{1}{4} \left( \frac{1}{2} \Delta_\phi b_3 \right)^{-1} z_{k+1}^{-\frac{1}{2}} - \frac{1}{8} - \frac{1}{2} b_3^{-1} \varepsilon(z_k^{-\frac{1}{2}}) z_k - \frac{1}{4} b_3^{-1} \sum_{r \in 2\N+5} r b_r z_k^{\frac{r-3}{2}} - (ck+2) \zeta(z_k),
\end{multline*}
where we used that $z \zeta'(z) \leq \zeta(z)$. As $\frac{\varepsilon(z^{-\frac{1}{2}}) z}{\zeta(z)} = \frac{\varepsilon(z^{-\frac{1}{2}})}{1-\Delta_\phi \sum_{r \in 2\N+3} b_r z^{\frac{r-2}{2}}}$ and $\frac{\sum_{r \in 2\N+5} r b_r z^{\frac{r-3}{2}}}{\zeta(z)} = \frac{\sum_{r \in 2\N+5} r b_r z^{\frac{r-5}{2}}}{1-\Delta_\phi \sum_{r \in 2\N+3} b_r z^{\frac{r-2}{2}}}$ are increasing in $z$, we have $\frac{\varepsilon(z_k^{-\frac{1}{2}}) z_k}{\zeta(z_k)} \leq \frac{\varepsilon(z_1^{-\frac{1}{2}}) z_1}{\zeta(z_1)}$ and $\frac{\sum_{r \in 2\N+5} r b_r z_k^{\frac{r-3}{2}}}{\zeta(z_k)} \leq \frac{\sum_{r \in 2\N+5} r b_r z_1^{\frac{r-3}{2}}}{\zeta(z_1)}$. Assuming that $c \geq 2 + \frac{1}{2} b_3^{-1} \frac{\varepsilon(z_1^{-\frac{1}{2}}) z_1}{\zeta(z_1)} + \frac{1}{4} b_3^{-1} \frac{\sum_{r \in 2\N+5} r b_r z_1^{\frac{r-3}{2}}}{\zeta(z_1)}$, we then have that $u_{k+1} \geq \frac{1}{4} (\frac{1}{2} \Delta_\phi b_3)^{-1} z_{k+1}^{-\frac{1}{2}} - \frac{1}{8} - c (k+1) z_{k+1}$, completing the induction. We get the claim since $z_k \leq O(\Delta_\phi^{-2} k^{-2})$ by Proposition~\ref{prop:inverse_cosine_distance_propagation}.
\end{proof}

We are going to use the above approximation result to analyze the spectrum of the limiting NTK matrix via the inverse cosine distance matrices, both defined below.

\begin{definition}[Limiting NTK matrix]~\\
Given a dataset $\{x_1,\ldots,x_n\} \subset \R^{m_0}$ of size $n \in \N+1$, the corresponding limiting NTK matrix $\limiting{K} \in \mathbb{S}^{n m_l}_+$ is defined blockwise as 
\[
\limiting{K} = \left[ \frac{1}{n} \limiting{K}(x_{i_1},x_{i_2}) : i_1,i_2 \in [1:n] \right].
\]
\end{definition}

\begin{definition}[Limiting inverse cosine distance matrices]~\\
Given a dataset $\{x_1,\ldots,x_n\} \subset \R^{m_0}$ of size $n \in \N+1$ and $k \in [1:l]$, define $\limiting{W}_k \in \mathbb{S}^n$ as
\[
{\limiting{W}_k}_{i,i} = 0 \text{ for } i \in [1:n] \text{ and } {\limiting{W}_k}_{i_1,i_2} = \left( \frac{1 - \limiting{\rho}_k(x_{i_1},x_{i_2})}{2} \right)^{-\frac{1}{2}} \text{ for } i_1 \neq i_2 \in [1:n].
\]
\end{definition}

\begin{proposition}[Spectral bounds for limiting inverse cosine distance matrices]\label{prop:inverse_cosine_distance_matrices}~\\
Given a dataset $x_1,\cdots,x_n \in \R^{m_0}$ with no parallel data points and setting \eqref{def:eoc_parameterization}, for all $k \in [1:l]$ there exists $W_k \in (1,\infty)$ such that $W_k = \Theta_{\limiting{W}_1}(1)$ and
\[
\left\Vert \limiting{W}_k - \omega^{\circ (k-1)}\left( W_k \right) \left( \mathbbm{1}_n^{\otimes 2} - \Id_n \right) \right\Vert \leq O(\Delta_\phi n^{-1} k) + O(1).
\]
\end{proposition}
\begin{proof}
First, we obtain spectral bounds for the matrices $\limiting{W}_k$ on the subspace $\mathbbm{1}_n^{\perp} = \{ y \in \R^n : \langle y, \mathbbm{1}_n \rangle = 0 \} \subset \R^n$. To this end, denote $\hat{\lambda}_{\min}(A) = \min_{y \in \mathbbm{1}_n^{\perp} : \Vert y \Vert = 1}\{ \langle y, A y \rangle \}$ and $\hat{\lambda}_{\max}(A) = \max_{y \in \mathbbm{1}_n^{\perp} : \Vert y \Vert = 1}\{ \langle y, A y \rangle \}$ for all $A \in \mathbb{S}^n$.

Define the matrix $\limiting{W} \in \mathbb{S}^n$ as follows. Let $\limiting{W}_{i,i} = 0$ for $i \in [1:n]$ and denote $w_* = (\frac{1-\cos(\min\{ \Delta_\phi^{-1} \frac{\pi}{2}, \pi\})}{2})^{-\frac{1}{2}}$. For $i_1 \neq i_2 \in [1:n]$, if ${\limiting{W}_1}_{i_1,i_2} \geq w_*$, then let $\limiting{W}_{i_1,i_2} = {\limiting{W}_1}_{i_1,i_2}$, while if ${\limiting{W}_1}_{i_1,i_2} < w_*$ then set $\limiting{W}_{i_1,i_2} > w_*$ such that $\omega(\limiting{W}_{i_1,i_2}) = \omega({\limiting{W}_1}_{i_1,i_2})$. Note that such a choice always exists uniquely as $\omega$ is continuous, strictly decreasing on $(1,w_*)$ from $\lim_{w \to 1^+} \omega(w) = (1-\Delta_\phi)^{-\frac{1}{2}}$ to $\omega(w_*)$ and strictly increasing on $(w_*,\infty)$ from $\omega(w_*)$ to $\infty$. Now let $\underline{W} = \min_{i_1 \neq i_2 \in [1:n]}\left\{ \limiting{W}_{i_1,i_2} \right\}$ and $\overline{W} = \max_{i_1 \neq i_2 \in [1:n]}\left\{ \limiting{W}_{i_1,i_2} \right\}$. With this construction, we then have for all $k \in [2:l]$ that ${\limiting{W}_k}_{i_1,i_2} = \omega^{\circ (k-1)}(\limiting{W}_{i_1,i_2}) = \omega^{\circ (k-1)}({\limiting{W}_1}_{i_1,i_2})$ with $\omega^{\circ (k-1)}$ being convex and strictly increasing on $[\underline{W},\infty)$ by Proposition~\ref{prop:inverse_cosine_distance_map}.

Letting $y \in \R^n$ such that $y_i = 0$ for $i \in [1:n]$ except for $y_{i_1} = -y_{i_2} = 2^{-\frac{1}{2}}$ if ${\limiting{W}_k}_{i_1,i_2} = \max_{i_1 \neq i_2 \in [1:n]}\left\{ {\limiting{W}_k}_{i_1,i_2} \right\} = \omega^{\circ (k-1)}(\overline{W})$, we have $y \in \mathbbm{1}_n^{\perp}$, $\Vert y \Vert = 1$ and $\langle y, \limiting{W}_k y \rangle = -\omega^{\circ (k-1)}(\overline{W})$, so that $\hat{\lambda}_{\min}(\limiting{W}) \leq -\overline{W}$ and $\hat{\lambda}_{\min}(\limiting{W}_k) \leq -\omega^{\circ (k-1)}(\overline{W})$ for all $k \in [2:l]$. An analogous argument shows that the bounds $\hat{\lambda}_{\max}(\limiting{W}_k) \geq -\underline{W}$ and $\hat{\lambda}_{\max}(\limiting{W}_k) \geq -\omega^{\circ (k-1)}(\underline{W})$ for all $k \in [2:l]$ hold as well. Defining $\widetilde{\limiting{W}}(w) \in \mathbb{S}^n$ for $w \geq w_*$ as $\widetilde{\limiting{W}}(w)_{i_1,i_2} = \max\{ \limiting{W}_{i_1,i_2}, w \}$ for $i_1 \neq i_2 \in [1:n]$ and $\widetilde{\limiting{W}}(w)_{i,i} = 0$ for $i \in [1:n]$, let
\[
\widetilde{W} = \max_{i_1 \neq i_2 \in [1:n]}\left\{ -\hat{\lambda}_{\min}\left( \widetilde{\limiting{W}}\left( \limiting{W}_{i_1,i_2} \right) \right) \right\}
\]
so that $\widetilde{W} \geq -\hat{\lambda}_{\min}(\limiting{W}) \geq \overline{W}$ and let
\[
\widetilde{\Delta} = \max_{i_1 \neq i_2 \in [1:n]}\left\{ \hat{\lambda}_{\max}\left( \widetilde{\limiting{W}}\left( \limiting{W}_{i_1,i_2} \right) \right) - \hat{\lambda}_{\min}\left( \widetilde{\limiting{W}}\left( \limiting{W}_{i_1,i_2} \right) \right) \right\}
\]
so that $\widetilde{\Delta} \geq \hat{\lambda}_{\max}(\limiting{W}) - \hat{\lambda}_{\min}(\limiting{W}) \geq \overline{W} - \underline{W}$.

Given $y \in \mathbbm{1}_n^{\perp}$ with $\Vert y \Vert = 1$, we have $\langle y, (\mathbbm{1}_n^{\otimes 2} - \Id_n) y \rangle = \sum_{i_1 = 1}^n \sum_{\substack{i_2=1 \\ i_2 \neq i_1}}^n y_{i_1} y_{i_2} = \sum_{i_1 = 1}^n y_{i_1} (\langle y, \mathbbm{1}_n \rangle - y_{i_1}) = -\Vert y \Vert^2 = -1$. By definition of $\widetilde{W}$, we have for all $y \in \mathbbm{1}_n^{\perp}$ with $\Vert y \Vert = 1$ that $-\langle y, \widetilde{\limiting{W}}(w) y \rangle - \widetilde{W} = -\sum_{i_1 = 1}^n \sum_{\substack{i_2=1 \\ i_2 \neq i_1}}^n y_{i_1} y_{i_2} \max\{ \limiting{W}_{i_1,i_2} - w, 0 \} - \max\{ \widetilde{W} - w, 0 \} \leq 0$ if either $w = \limiting{W}_{i_1,i_2}$ for some $i_1 \neq i_2 \in [1:n]$ or $w = \widetilde{W}$. Therefore, as ${\limiting{W}_j}_{i_1,i_2} = \omega^{\circ (j-1)}(\limiting{W}_{i_1,i_2})$ for $i_1 \neq i_2 \in [1:n]$ and $\omega^{\circ (j-1)}$ is a nonnegative increasing convex function on $[\underline{W},\infty)$ by Proposition~\ref{prop:inverse_cosine_distance_map}, we can apply \citet[Theorem~9(a\textsubscript{2})]{Horvath2023} to get the bound $-\langle y, \limiting{W}_k y \rangle \leq \omega^{\circ (k-1)}(\widetilde{W})$. As this works for all $y \in \mathbbm{1}_n^{\perp}$ with $\Vert y \Vert = 1$, we have $\hat{\lambda}_{\min}\left( \limiting{W}_k \right) \geq -\omega^{\circ (k-1)}(\widetilde{W})$.

Given $y_1,y_2 \in \mathbbm{1}_n^{\perp}$ with $\Vert y_1 \Vert = \Vert y_2 \Vert = 1$, we have $\langle y_1, (\mathbbm{1}_n^{\otimes 2} - \Id_n) y_1 \rangle - \langle y_2, (\mathbbm{1}_n^{\otimes 2} - \Id_n) y_2 \rangle = 0$. By definition of $\widetilde{\Delta}$, we have for all $y_1,y_2 \in \mathbbm{1}_n^{\perp}$ with $\Vert y_1 \Vert = \Vert y_2 \Vert = 1$ that $\langle y_1, \widetilde{\limiting{W}}(w) y_1 \rangle - \langle y_2, \widetilde{\limiting{W}}(w) y_2 \rangle - ((\overline{W} + \widetilde{\Delta}) - \overline{W}) = \sum_{i_1 = 1}^n \sum_{\substack{i_2=1 \\ i_2 \neq i_1}}^n {y_1}_{i_1} {y_1}_{i_2} \max\{ \limiting{W}_{i_1,i_2} - w, 0 \} - \sum_{i_1 = 1}^n \sum_{\substack{i_2=1 \\ i_2 \neq i_1}}^n {y_2}_{i_1} {y_2}_{i_2} \max\{ \limiting{W}_{i_1,i_2} - w, 0 \} - (\max\{ \overline{W} + \widetilde{\Delta} - w, 0 \} - \max\{ \overline{W} - w, 0 \}) \leq 0$ if either $w = \limiting{W}_{i_1,i_2}$ for some $i_1 \neq i_2 \in [1:n]$ or $w = \overline{W} + \widetilde{\Delta}$. Via \citet[Theorem~9(a\textsubscript{2})]{Horvath2023}, we then get the bound $\langle y_1, \limiting{W}_k y_1 \rangle - \langle y_2, \limiting{W}_k y_2 \rangle \leq \omega^{\circ (k-1)}(\overline{W} + \widetilde{\Delta}) - \omega^{\circ (k-1)}(\overline{W})$. As this works for all $y_1,y_2 \in \mathbbm{1}_n^{\perp}$ with $\Vert y_1 \Vert = \Vert y_2 \Vert = 1$, we have $\hat{\lambda}_{\max}\left( \limiting{W}_k \right) - \hat{\lambda}_{\min}\left( \limiting{W}_k \right) \leq \omega^{\circ (k-1)}(\overline{W} + \widetilde{\Delta}) - \omega^{\circ (k-1)}(\overline{W}) \leq \widetilde{\Delta}$, using that $\omega^{\circ (k-1)}$ is $1$-Lipschitz by Proposition~\ref{prop:inverse_cosine_distance_map}. Therefore $\hat{\lambda}_{\max}(\limiting{W}_k) \leq \hat{\lambda}_{\min}(\limiting{W}_k) + \widetilde{\Delta} \leq -\omega^{\circ (k-1)}(\overline{W}) + \widetilde{\Delta}$. By the min-max theorem, we then have the bound $\lambda_2(\limiting{W}_k) \leq -\omega^{\circ (k-1)}(\overline{W}) + \widetilde{\Delta}$.

Hence for any $y \in \mathbbm{1}_n^{\perp}$ the bounds $\hat{\lambda}_{\min}(\limiting{W}_k) \Vert y \Vert^2 \leq \langle y, \limiting{W}_k y \rangle \leq \hat{\lambda}_{\max}(\limiting{W}_k) \Vert y \Vert^2$ hold, so that for $k$ large enough we have $\Vert \limiting{W}_k y \Vert \leq -\hat{\lambda}_{\min}(\limiting{W}_k) \Vert y \Vert \leq \omega^{\circ (k-1)}(\widetilde{W}) \Vert y \Vert$ for any $y \in \mathbbm{1}_n^{\perp}$.

Denoting $\hat{y} = (\Id_n - \mathbbm{1}_n^{\otimes 2}) y = y - \langle y, \frac{1}{\sqrt{n}} \mathbbm{1}_n \rangle \frac{1}{\sqrt{n}} \mathbbm{1}_n \in \mathbbm{1}_n^{\perp}$ for $y \in \R^n$, we have $\Vert \hat{y} \Vert = \sqrt{\Vert y \Vert^2 - \langle y, \frac{1}{\sqrt{n}} \mathbbm{1}_n \rangle^2}$ and $y = \hat{y} + \langle y, \frac{1}{\sqrt{n}} \mathbbm{1}_n \rangle \frac{1}{\sqrt{n}} \mathbbm{1}_n$. We can then write
\[
\left\langle y, \limiting{W}_k y \right\rangle 
= \left\langle y, \frac{1}{\sqrt{n}} \mathbbm{1}_n \right\rangle^2 \left\langle \frac{1}{\sqrt{n}} \mathbbm{1}_n, \limiting{W}_k \frac{1}{\sqrt{n}} \mathbbm{1}_n \right\rangle + 2 \left\langle y, \frac{1}{\sqrt{n}} \mathbbm{1}_n \right\rangle \left\langle \frac{1}{\sqrt{n}} \mathbbm{1}_n, \limiting{W}_k \hat{y} \right\rangle + \left\langle \hat{y}, \limiting{W}_k \hat{y} \right\rangle.
\]
By Jensen's inequality, we have $\langle \frac{1}{\sqrt{n}} \mathbbm{1}_n, \limiting{W}_k \frac{1}{\sqrt{n}} \mathbbm{1}_n \rangle \geq (n-1) \omega^{\circ (k-1)}(\widehat{W})$ where we denoted $\widehat{W} = \frac{1}{n(n-1)} \langle \mathbbm{1}_n, \limiting{W} \mathbbm{1}_n \rangle$. Note that the bounds
\[
\left\vert \left\langle \frac{1}{\sqrt{n}} \mathbbm{1}_n, \limiting{W}_k \hat{y} \right\rangle \right\vert 
\leq \left\Vert \limiting{W}_k \hat{y} \right\Vert 
\leq -\hat{\lambda}_{\min}\left( \limiting{W}_k \right) \sqrt{\Vert y \Vert^2 - \left\langle y, \frac{1}{\sqrt{n}} \mathbbm{1}_n \right\rangle^2}.
\]
and $\langle \hat{y}, \limiting{W}_k \hat{y} \rangle \geq \hat{\lambda}_{\min}(\limiting{W}_k) \Vert \hat{y} \Vert^2 = \hat{\lambda}_{\min}(\limiting{W}_k) (\Vert y \Vert^2 - \langle y, \frac{1}{\sqrt{n}} \mathbbm{1}_n \rangle^2)$ both hold as well. Letting $\rho = \vert \langle  y, \frac{1}{\sqrt{n}} \mathbbm{1}_n \rangle \vert \in [0,1]$ for $y \in \R^n$ such that $\Vert y \Vert = 1$, we therefore have that
\[
\langle y, \limiting{W}_k y \rangle
\geq \rho^2 (n-1) \omega^{\circ (k-1)}\left( \widehat{W} \right)
- 2 \rho \sqrt{1 - \rho^2} \omega^{\circ (k-1)}\left( \widetilde{W} \right)
- (1 - \rho^2) \omega^{\circ (k-1)}\left( \widetilde{W} \right).
\]
The minimum of this with respect to $\rho$ equals
\begin{multline*}
-\frac{1}{2} \left( \sqrt{\left( (n-1) \omega^{\circ (k-1)}\left( \widehat{W} \right) + \omega^{\circ (k-1)}\left( \widetilde{W} \right) \right)^2 + \left( 2 \omega^{\circ (k-1)}\left( \widetilde{W} \right) \right)^2}
\right. \\ \left.
- \sqrt{\left( (n-1) \omega^{\circ (k-1)}\left( \widehat{W} \right) - \omega^{\circ (k-1)}\left( \widetilde{W} \right) \right)^2} \right),
\end{multline*}
which is at least $-\frac{1}{4} ((n-1) \omega^{\circ (k-1)}(\widehat{W}) - \omega^{\circ (k-1)}(\widetilde{W}))^{-1} (((n-1) \omega^{\circ (k-1)}(\widehat{W}) + \omega^{\circ (k-1)}(\widetilde{W}))^2 + (2 \omega^{\circ (k-1)}(\widetilde{W}))^2 - ((n-1) \omega^{\circ (k-1)}(\widehat{W}) - \omega^{\circ (k-1)}(\widetilde{W}))^2)$ using that $\sqrt{s_1} - \sqrt{s_2} \leq \frac{1}{2} s_2^{-\frac{1}{2}} (s_1 - s_2)$ for any $s_1 > s_2 > 0$ via the fundamental theorem of calculus. This equals
\[
-\omega^{\circ (k-1)}\left( \widetilde{W} \right) \left( 1 + 2 (n-1)^{-1}  \left( \frac{\omega^{\circ (k-1)}\left( \widehat{W} \right)}{\omega^{\circ (k-1)}\left( \widetilde{W} \right)} - (n-1)^{-1} \right)^{-1} \right).
\]
By Proposition~\ref{prop:inverse_cosine_distance_propagation}, we have $(\frac{\omega^{\circ (k-1)}(\widehat{W})}{\omega^{\circ (k-1)}(\widetilde{W})} - (n-1)^{-1})^{-1} \leq O(1)$, so that 
\[
\left\langle y, \limiting{W}_k y \right\rangle 
\geq -\omega^{\circ (k-1)}\left( \widetilde{W} \right) (1 + O(n^{-1}))
\geq -\omega^{\circ (k-1)}\left( \widetilde{W} \right) - O(\Delta_\phi n^{-1} k)
\]
for all $y \in \R^n$ such that $\Vert y \Vert = 1$, implying that $\lambda_n(\limiting{W}_k) \geq -\omega^{\circ (k-1)}(\widetilde{W}) - O(\Delta_\phi n^{-1} k)$.

By the Perron-Frobenius theorem, it holds that
\[
(n-1) \omega^{\circ (k-1)}\left( \underline{W} \right)
\leq \lambda_1\left( \limiting{W}_k \right)
\leq (n-1) \omega^{\circ (k-1)}\left( \overline{W} \right),
\]
so that as $\omega^{\circ (k-1)}$ is continuous and strictly increasing on $[\underline{W},\overline{W}]$, there exists $W_k \in [\underline{W},\overline{W}]$ such that $\lambda_1(\limiting{W}_k) = (n-1) \omega^{\circ (k-1)}(W_k)$. Noting that we have $\lambda_2(\limiting{W}_k) \leq -\omega^{\circ (k-1)}(\overline{W}) + \widetilde{\Delta} \leq -\omega^{\circ (k-1)}(W_k) + \widetilde{\Delta} + \vert \overline{W} - W_k \vert \leq -\omega^{\circ (k-1)}(W_k) + O(1)$ and $\lambda_n(\limiting{W}_k) \geq -\omega^{\circ (k-1)}(\widetilde{W}) - O(\Delta_\phi n^{-1} k) \geq -\omega^{\circ (k-1)}(W_k) - O(\Delta_\phi n^{-1} k) - \vert \widetilde{W} - W_k \vert \geq -\omega^{\circ (k-1)}(W_k) - O(\Delta_\phi n^{-1} k) - O(1)$ as $\omega^{\circ (k-1)}$ is $1$-Lipschitz by Proposition~\ref{prop:inverse_cosine_distance_map}, we then have the bound
\[
\left\Vert \limiting{W}_k - \omega^{\circ (k-1)}(W_k) (\mathbbm{1}_n^{\otimes 2} - \Id_n) \right\Vert \leq O(\Delta_\phi n^{-1} k) + O(1)
\]
since $\lambda_1(\mathbbm{1}_n^{\otimes 2} - \Id_n) = n - 1$ and $\lambda_i(\mathbbm{1}_n^{\otimes 2} - \Id_n) = -1$ for $i \in [2:n]$.
\end{proof}

The asymptotic spectral behaviour of $\limiting{K}$ in the large depth limit was analyzed by \citet{Xiaoetal2020}. The following proposition quantifies the limiting behaviour of $\limiting{K}$ as the depth $l$ grows, giving a quantitative generalization of the results in \citet[Appendix~C.1]{Xiaoetal2020}.

\begin{theorem}[The spectrum of $\limiting{K}$ at the EOC]\label{thm:ntk_spectrum}~\\
Given a dataset $x_1,\cdots,x_n \in \R^{m_0}$ with no parallel data points, setting \eqref{def:eoc_parameterization} and denoting $\limiting{\tau} = [\Vert x_i \Vert : i \in [1:n]]$, $\underline{\tau}=\min_{i\in[1:n]}\{\limiting{\tau}_i\}$ and $\overline{\tau}=\max_{i\in[1:n]}\{\limiting{\tau}_i\}$, there exists $W \in (1,\infty)$ such that $W = \Theta_{\limiting{W}_1}(1)$ and denoting
\[
\xi = \frac{3}{8} \left( \Delta_\phi^{-1} \frac{\pi}{2} W + \log\left( \Delta_\phi^{-1} \frac{3\pi}{4} W + l - 1 \right) - 1 \right)
\]
we have the spectral bounds
\[
\lambda_1\left( \limiting{K} \right) 
\leq \overline{\tau}^2 \left( \left( 1 + \frac{3}{n} \right) \frac{1}{4} l + \left( 1 - \frac{1}{n} \right) \xi \right) + O\left( \overline{\tau}^2 \left( \Delta_\phi^{-1} + \Delta_\phi^{-2} l^{-1} + n^{-2} l \right) \right),
\]
\[
\lambda_{m_l}\left( \limiting{K} \right)
\geq \underline{\tau}^2 \left( \left( 1 + \frac{3}{n} \right) \frac{1}{4} l + \left( 1 - \frac{1}{n} \right) \xi \right) - O\left( \overline{\tau}^2 \left( \Delta_\phi^{-1} + \Delta_\phi^{-2} l^{-1} + n^{-2} l \right) \right),
\]
\[
\lambda_{m_l+1}\left( \limiting{K} \right)
\leq \overline{\tau}^2 \frac{1}{n} \left( \frac{3}{4} l - \xi \right) + O\left( \overline{\tau}^2 \left( \Delta_\phi^{-1} n^{-1} + \Delta_\phi^{-2} l^{-1} + n^{-2} l \right) \right)
\]
and
\[
\lambda_{n m_l}\left( \limiting{K} \right)
\geq \underline{\tau}^2 \frac{1}{n} \left( \frac{3}{4} l - \xi \right) - O\left( \overline{\tau}^2 \left( \Delta_\phi^{-1} n^{-1} + \Delta_\phi^{-2} l^{-1} + n^{-2} l \right) \right),
\]
while the condition number satisfies
\begin{multline*}
\frac{\underline{\tau}^2}{\overline{\tau}^2} \left( 1 + \frac{1}{3} n + \frac{16}{9} n \frac{\xi}{l - \frac{4}{3} \xi} \right) - O\left( \frac{\underline{\tau}^2}{\overline{\tau}^2} \left( 1 + \Delta_\phi^{-1} n l^{-1} + \Delta_\phi^{-2} n^2 l^{-2} \right) \right) \\
\leq \kappa\left( \limiting{K} \right) = \frac{\lambda_1\left( \limiting{K} \right)}{\lambda_{n m_l}\left( \limiting{K} \right)} \\
\leq \frac{\overline{\tau}^2}{\underline{\tau}^2} \left( 1 + \frac{1}{3} n + \frac{16}{9} n \frac{\xi}{l - \frac{4}{3} \xi} \right) + O\left( \frac{\overline{\tau}^2}{\underline{\tau}^2} \left( 1 + \Delta_\phi^{-1} n l^{-1} + \Delta_\phi^{-2} n^2 l^{-2} \right) \right).
\end{multline*}
\end{theorem}
\begin{proof}
Denoting $\widetilde{U} = \frac{1}{2} \Delta_\phi^{-1} b_3^{-1} \limiting{W}_l - \frac{1}{8} (\mathbbm{1}_n^{\otimes 2} - \Id_n)$, we have 
\begin{multline*}
\left\Vert \widetilde{U} - \left( \frac{1}{2} \Delta_\phi^{-1} b_3^{-1} \omega^{\circ (l-1)}\left( W \right) - \frac{1}{8} \right) \left( \mathbbm{1}_n^{\otimes 2} - \Id_n \right) \right\Vert \\
\leq \frac{1}{2} \Delta_\phi^{-1} b_3^{-1} \left\Vert \limiting{W}_l - \omega^{\circ (l-1)}\left( W \right) \left( \mathbbm{1}_n^{\otimes 2} - \Id_n \right) \right\Vert \leq O(n^{-1} l) + O(\Delta_\phi^{-1})
\end{multline*}
by Proposition~\ref{prop:inverse_cosine_distance_matrices}. Let 
\[
U = \left[ \sum_{k=1}^l \varrho^{\circ (k-1)}\left( \limiting{\rho}_1(x_{i_1},x_{i_2}) \right) \prod_{k'=k}^{l-1} \varrho'\left( \varrho^{\circ (k'-1)}\left( \limiting{\rho}_1(x_{i_1},x_{i_2}) \right) \right) : i_1,i_2 \in [1:n] \right],
\]
so that $U_{i,i} = l$ for $i \in [1:n]$ and $\limiting{K} = \left( \frac{1}{n} D_{\limiting{\tau}} U D_{\limiting{\tau}} \right) \boxtimes \Id_{m_l}$ by Proposition~\ref{prop:limiting_ntk_at_eoc}. By Proposition~\ref{prop:inverse_cosine_distances_approximate_ntk} we get $\Vert U - (\widetilde{U} + l \Id_n) \Vert \leq \Vert U - (\widetilde{U} + l \Id_n) \Vert_\infty \leq O(\Delta_\phi^{-2} n l^{-1})$, so that by the triangle inequality and submultiplicativity we have the bound
\[
\left\Vert \frac{1}{n} D_{\limiting{\tau}} U D_{\limiting{\tau}} - \frac{1}{n} D_{\limiting{\tau}} \left( \left( \frac{1}{2} \Delta_\phi^{-1} b_3^{-1} \omega^{\circ (l-1)}\left( W \right) - \frac{1}{8} \right) \left( \mathbbm{1}_n^{\otimes 2} - \Id_n \right) + l \Id_n \right) D_{\limiting{\tau}} \right\Vert 
\leq \epsilon
\]
with $\epsilon \leq \overline{\tau}^2 (O(\Delta_\phi^{-2} l^{-1}) + O(n^{-2} l) + O(\Delta_\phi^{-1} n^{-1}))$. Letting $c = l^{-1} ( \frac{1}{2} \Delta_\phi^{-1} b_3^{-1} \omega^{\circ (l-1)}( W ) - \frac{1}{8} )$, the matrix on the right can be written as $\frac{1}{n} D_{\limiting{\tau}} ( c l \mathbbm{1}_n^{\otimes 2} + (1 - c) l \Id_n ) D_{\limiting{\tau}} = \frac{l}{n} ((1-c) D_{\limiting{\tau}}^2 + c \limiting{\tau}^{\otimes 2})$, showing that it is a rank one perturbation of a diagonal matrix. By \citet[\S~1]{Jakovcevicetal2015}, we get the spectral bounds $\lambda_2(\frac{l}{n} ((1-c) D_{\limiting{\tau}}^2 + c \limiting{\tau}^{\otimes 2})) \leq \frac{l}{n} (1-c) \overline{\tau}^2$ and $\lambda_n(\frac{l}{n} ((1-c) D_{\limiting{\tau}}^2 + c \limiting{\tau}^{\otimes 2})) \geq \frac{l}{n} (1-c) \underline{\tau}^2$, while the Perron-Frobenius theorem gives that $\frac{l}{n} (1 + (n-1) c) \underline{\tau}^2 \leq \lambda_1(\frac{l}{n} ((1-c) D_{\limiting{\tau}}^2 + c \limiting{\tau}^{\otimes 2})) \leq \frac{l}{n} (1 + (n-1) c) \overline{\tau}^2$. As Proposition~\ref{prop:inverse_cosine_distance_propagation} implies
\begin{multline*}
c = l^{-1} \left( \frac{1}{2} \Delta_\phi^{-1} b_3^{-1} \left( W + \frac{1}{2} \Delta_\phi b_3 (l-1) + \frac{3}{4} \Delta_\phi b_3 \log\left( 2 \Delta_\phi^{-1} b_3^{-1} W + l - 1 \right) \pm O(1) \right) - \frac{1}{8} \right) \\
= \frac{1}{4} + l^{-1} \xi \pm O(\Delta_\phi^{-1} l^{-1}),
\end{multline*}
we have
\[
\lambda_2\left( \frac{l}{n} \left( (1-c) D_{\limiting{\tau}}^2 + c \limiting{\tau}^{\otimes 2} \right) \right)
\leq \overline{\tau}^2 \left( \frac{3}{4} \frac{l}{n} - n^{-1} \xi + O(\Delta_\phi^{-1} n^{-1}) \right),
\]
\[
\lambda_n\left( \frac{l}{n} \left( (1-c) D_{\limiting{\tau}}^2 + c \limiting{\tau}^{\otimes 2} \right) \right)
\geq \underline{\tau}^2 \left( \frac{3}{4} \frac{l}{n} - n^{-1} \xi - O(\Delta_\phi^{-1} n^{-1}) \right),
\]
\[
\lambda_1\left( \frac{l}{n} \left( (1-c) D_{\limiting{\tau}}^2 + c \limiting{\tau}^{\otimes 2} \right) \right)
\leq \overline{\tau}^2 \left( \left( 1 + \frac{3}{n} \right) \frac{1}{4} l + \left( 1 - \frac{1}{n} \right) \xi + O(\Delta_\phi^{-1}) \right) 
\]
and
\[
\lambda_1\left( \frac{l}{n} \left( (1-c) D_{\limiting{\tau}}^2 + c \limiting{\tau}^{\otimes 2} \right) \right)
\geq \underline{\tau}^2 \left( \left( 1 + \frac{3}{n} \right) \frac{1}{4} l + \left( 1 - \frac{1}{n} \right) \xi - O(\Delta_\phi^{-1}) \right).
\]
By Weyl's inequality, these bounds hold for $\frac{1}{n} D_{\limiting{\tau}} U D_{\limiting{\tau}}$ with an additional $\epsilon$ error. We get the eigenvalue bounds using the fact that the eigenvalues of $\limiting{K} = \left( \frac{1}{n} D_{\limiting{\tau}} U D_{\limiting{\tau}} \right) \boxtimes \Id_{m_l}$ are exactly the eigenvalues of $\frac{1}{n} D_{\limiting{\tau}} U D_{\limiting{\tau}}$, each repeated $m_l$ times.

Note that we then have via the fundamental theorem of calculus that
\begin{multline*}
\underline{\tau}^2 \lambda_n\left( \frac{1}{n} D_{\limiting{\tau}} U D_{\limiting{\tau}} \right)^{-1} - \left( \frac{3}{4} \frac{l}{n} - n^{-1} \xi \right)^{-1} \\
\leq \left( \frac{3}{4} \frac{l}{n} - n^{-1} \xi - O(\Delta_\phi^{-1} n^{-1}) - \epsilon \right)^{-2} (O(\Delta_\phi^{-1} n^{-1}) + \epsilon) 
\leq O(n^2 l^{-2}) (O(\Delta_\phi^{-1} n^{-1}) + \epsilon) 
\end{multline*}
and
\begin{multline*}
\overline{\tau}^2 \lambda_2\left( \frac{1}{n} D_{\limiting{\tau}} U D_{\limiting{\tau}} \right)^{-1} - \left( \frac{3}{4} \frac{l}{n} - n^{-1} \xi \right)^{-1} \\
\geq \left( \frac{3}{4} \frac{l}{n} - n^{-1} \xi + O(\Delta_\phi^{-1} n^{-1}) + \epsilon \right)^{-2} (O(\Delta_\phi^{-1} n^{-1}) + \epsilon)
\leq O(n^2 l^{-2}) (O(\Delta_\phi^{-1} n^{-1}) + \epsilon).
\end{multline*}
Therefore $\kappa\left( \frac{1}{n} D_{\limiting{\tau}} U D_{\limiting{\tau}} \right)$ is at most
\begin{multline*}
\frac{\overline{\tau}^2 \left( \left( 1 + \frac{3}{n} \right) \frac{1}{4} l + \left( 1 - \frac{1}{n} \right) \xi + O(\Delta_\phi^{-1}) \right) + \epsilon}{\underline{\tau}^2 (\frac{3}{4} \frac{l}{n} - n^{-1} \xi)} + O(\overline{\tau}^2 l) O(\underline{\tau}^{-2} n^2 l^{-2}) (O(\Delta_\phi^{-1} n^{-1}) + \epsilon) \\
\leq \frac{\overline{\tau}^2}{\underline{\tau}^2} \left( 1 + \frac{1}{3} n + \frac{16}{9} n \frac{\xi}{l - \frac{4}{3} \xi} \right) + O\left( \frac{\overline{\tau}^2}{\underline{\tau}^2} \left( 1 + \Delta_\phi^{-1} n l^{-1} + \Delta_\phi^{-2} n^2 l^{-2} \right) \right)
\end{multline*}
and at least
\begin{multline*}
\frac{\overline{\tau}^2 \left( \left( 1 + \frac{3}{n} \right) \frac{1}{4} l + \left( 1 - \frac{1}{n} \right) \xi - O(\Delta_\phi^{-1}) \right) - \epsilon}{\frac{l}{n} (1-c) \overline{\tau}^2} - O(\overline{\tau}^2 l) O(\underline{\tau}^{-2} n^2 l^{-2}) (O(\Delta_\phi^{-1} n^{-1}) + \epsilon) \\
\geq \frac{\underline{\tau}^2}{\overline{\tau}^2} \left( 1 + \frac{1}{3} n + \frac{16}{9} n \frac{\xi}{l - \frac{4}{3} \xi} \right) - O\left( \frac{\underline{\tau}^2}{\overline{\tau}^2} \left( 1 + \Delta_\phi^{-1} n l^{-1} + \Delta_\phi^{-2} n^2 l^{-2} \right) \right),
\end{multline*}
giving the condition number bounds and finishing the proof.
\end{proof}

\begin{remark}[Nondegeneracy]~\\
In order for the above results to be applicable, pairs of distinct data points cannot be parallel. In case there are parallel data points but no repeated ones, we can replace $x_i$ for all $i \in [1:n]$ with $[x_i,\beta]$ for some $\beta > 0$ to get a dataset without parallel pairs of data points. In particular, this is strictly necessary when dealing with $1$-dimensional data. This is equivalent to having a bias in the first layer.
\end{remark}

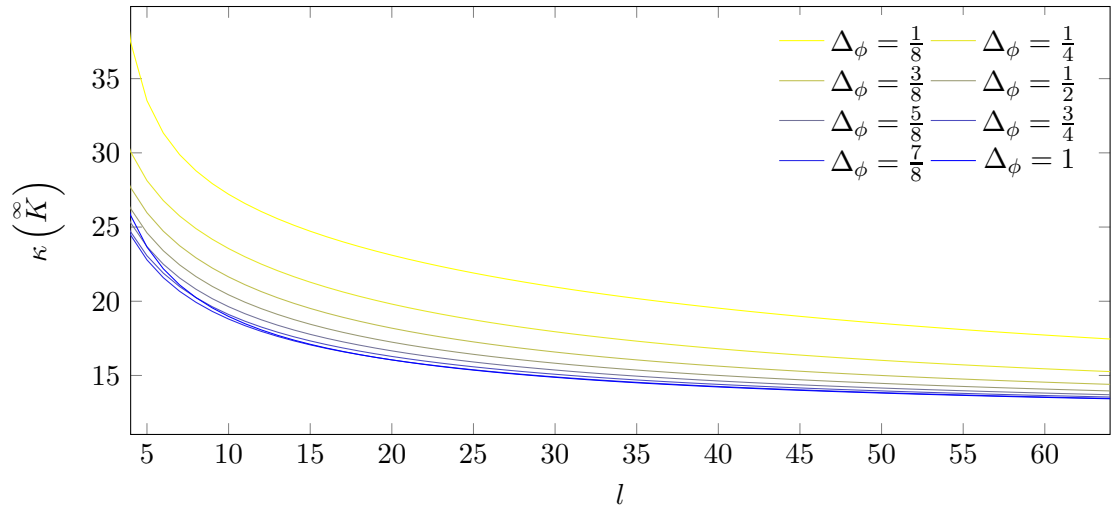
\begin{figure}[t]
\begin{center}
\centerline{
\begin{tikzpicture}
\pgfplotsset{cycle list/Dark2}
\begin{axis}[
    xlabel = $l$,
    ylabel = $\kappa\left( \limiting{K} \right)$,
    legend columns=2,
    legend style={draw=none},
    xmin=4,
    xmax=64,
	x post scale=1.9
]
\addlegendentry{$\Delta_\phi = \frac{1}{8}$}
\addplot[mycolor1] table [x=Step, y=Value, col sep=comma] {csvs/ntk/run-ntk_init_p_1_q_0_r_1_gaussian_d_32_m_1_J_64_a_2.6457513110645907_b_1.0_relu_samples_100_limiting_samples_1048576_2024-04-27_12_29_45_404571-tag-limiting_ntk_condition_number_mean.csv};
\addlegendentry{$\Delta_\phi = \frac{1}{4}$}
\addplot[mycolor2] table [x=Step, y=Value, col sep=comma] {csvs/ntk/run-ntk_init_p_1_q_0_r_1_gaussian_d_32_m_1_J_64_a_1.7320508075688772_b_1.0_relu_samples_100_limiting_samples_1048576_2024-04-27_12_29_24_400446-tag-limiting_ntk_condition_number_mean.csv};
\addlegendentry{$\Delta_\phi = \frac{3}{8}$}
\addplot[mycolor3] table [x=Step, y=Value, col sep=comma] {csvs/ntk/run-ntk_init_p_1_q_0_r_1_gaussian_d_32_m_1_J_64_a_2.23606797749979_b_1.7320508075688772_relu_samples_100_limiting_samples_1048576_2024-04-27_12_29_14_240406-tag-limiting_ntk_condition_number_mean.csv};
\addlegendentry{$\Delta_\phi = \frac{1}{2}$}
\addplot[mycolor4] table [x=Step, y=Value, col sep=comma] {csvs/ntk/run-ntk_init_p_1_q_0_r_1_gaussian_d_32_m_1_J_64_a_1.0_b_1.0_relu_samples_100_limiting_samples_1048576_2024-04-27_12_34_00_608747-tag-limiting_ntk_condition_number_mean.csv};
\addlegendentry{$\Delta_\phi = \frac{5}{8}$}
\addplot[mycolor5] table [x=Step, y=Value, col sep=comma] {csvs/ntk/run-ntk_init_p_1_q_0_r_1_gaussian_d_32_m_1_J_64_a_1.7320508075688772_b_2.23606797749979_relu_samples_100_limiting_samples_1048576_2024-04-27_12_34_22_402539-tag-limiting_ntk_condition_number_mean.csv};
\addlegendentry{$\Delta_\phi = \frac{3}{4}$}
\addplot[mycolor6] table [x=Step, y=Value, col sep=comma] {csvs/ntk/run-ntk_init_p_1_q_0_r_1_gaussian_d_32_m_1_J_64_a_1.0_b_1.7320508075688772_relu_samples_100_limiting_samples_1048576_2024-04-27_12_34_12_088109-tag-limiting_ntk_condition_number_mean.csv};
\addlegendentry{$\Delta_\phi = \frac{7}{8}$}
\addplot[mycolor7] table [x=Step, y=Value, col sep=comma] {csvs/ntk/run-ntk_init_p_1_q_0_r_1_gaussian_d_32_m_1_J_64_a_1.0_b_2.6457513110645907_relu_samples_100_limiting_samples_1048576_2024-04-27_12_34_07_320343-tag-limiting_ntk_condition_number_mean.csv};
\addlegendentry{$\Delta_\phi = 1$}
\addplot[mycolor8] table [x=Step, y=Value, col sep=comma] {csvs/ntk/run-ntk_init_p_1_q_0_r_1_gaussian_d_32_m_1_J_64_a_0.0_b_1.0_relu_samples_100_limiting_samples_1048576_2024-04-27_12_18_02_451927-tag-limiting_ntk_condition_number_mean.csv};
\end{axis}
\end{tikzpicture}
}
\caption{\label{fig:limit-behaviour}Condition number of $\limiting{K}$ for $(a,b)$-ReLUs as a function of depth for different values of $\Delta_\phi$ with the EOC parameterization \eqref{def:eoc_parameterization}. Each curve is the average of $100$ runs, where in each run we sampled $n=32$ points $x_1,\cdots,x_{32}$ uniformly from the unit sphere in $\R^{16}$ to form the dataset and calculated the corresponding $\limiting{K}$ in closed form via Proposition~\ref{prop:cosine_map} and Proposition~\ref{prop:limiting_ntk_at_eoc}.}
\label{plot:limiting_ntk_conditioning}
\end{center}
\end{figure}

\begin{remark}[Spectral scaling of $\limiting{K}$ at the EOC]\label{remark:limiting_ntk_spectral_scaling}~\\
Theorem~\ref{thm:ntk_spectrum} quantifies the behavior of the spectrum of $\limiting{K}$ in terms of depth. Note that having $\overline{\tau}$ close to $\underline{\tau}$ decreases the condition number of $\limiting{K}$. The fact that $\xi$ and the error terms scale with $\Delta_\phi^{-1}$ quantifies the way the $(a,b)$ parameters influence the conditioning of $\limiting{K}$. In particular, for $\phi_1(s) = a_1 s + b_1 \vert s \vert$ and $\phi_2(s) = a_2 s + b_2 \vert s \vert$ such that $\vert a_1 \vert \leq \vert a_2 \vert$ and $\vert b_1 \vert \geq \vert b_2 \vert$ so that $\Delta_{\phi_1} = \frac{b_1^2}{a_1^2+b_1^2} \geq \Delta_{\phi_1} = \frac{b_2^2}{a_2^2+b_2^2}$, less layers (smaller $l$) are sufficient to achieve the same conditioning with $\phi_1$ than with $\phi_2$. Informally, $\phi_1$ is ``more nonlinear'' than $\phi_2$ and having ``more nonlinearity'' by tuning $\vert a \vert$ down and $\vert b \vert$ up is beneficial for the conditioning of $\limiting{K}$. In particular, for this sole purpose the optimal choice is $a=0$ and $b \neq 0$. Figure~\ref{plot:limiting_ntk_conditioning} supports this argument empirically, while also showing that $\kappa(\limiting{K})$ converges to the limit $1 + \frac{1}{3} n$ as $l$ is increased. Last but not least, increasing the dataset size $n$ always results in worse conditioning.
\end{remark}

\section{Limitations and Future Directions} \label{limitations}
The main limitation of our results is that they cover only MLPs, a basic neural network architecture with a narrow range of practical applicability in real-world problems. We expect that studying the NTK via the approximate equivalence of its entries and the inverse cosine distances should be viable with other architectures such as convolutional neural networks and transformers. Another limitation of our work is that we only treat homogeneous activations. Extending our theory to other activations such as smooth rectifiers and sigmoid-like functions should be possible, but more difficult since without homogeneity the variances in the limiting inner products do not factor through the dual functions as in Proposition~\ref{prop:limiting_inner_products}. Additionally, the results of \citet{Xiaoetal2020} concerning smooth activations suggest that the corresponding NTKs should have worse conditioning than the ones we study here. Finally, while our result can readily be applied to study the training of MLPs in the kernel regime by exploiting the lazy training phenomenon, we believe that the most important future direction is the study of the behavior of the NTK matrix during training in the feature learning regime, where lazy training is absent and the NTK evolves in a nontrivial manner.

\section*{Acknowledgements}
D\'avid Terj\'ek and Diego Gonz\'alez-S\'anchez were supported by the Ministry of Innovation and Technology NRDI Office within the framework of the Artificial Intelligence National Laboratory (RRF-2.3.1-21-2022-00004).

\newpage

\bibliography{jmlr_submission}


\end{document}